\documentclass[a4paper,fleqn]{cas-sc}

\usepackage[numbers,sort&compress]{natbib}
\usepackage{amsthm,amsmath,amssymb}
\usepackage{mathtools,mathrsfs}
\usepackage{enumitem}
\usepackage{placeins}
\usepackage{tabularx}
\usepackage{algorithm}
\usepackage{algpseudocode}
\usepackage{xcolor}
\usepackage{tikz}
\usetikzlibrary{arrows.meta,positioning}

\theoremstyle{plain}
\newtheorem{theorem}{Theorem}[section]

\theoremstyle{definition}
\newtheorem{assumption}[theorem]{Assumption}

\newtheorem{remark}[theorem]{Remark}
\newtheorem{example}{Example}[section]

\allowdisplaybreaks[4]
\setlist[itemize]{leftmargin=2.2em,itemsep=.32em,topsep=.42em}
\setlist[enumerate]{leftmargin=2.5em,itemsep=.32em,topsep=.42em}
\newcommand{\R}{\mathbb{R}}
\newcommand{\dd}{\,\mathrm{d}}
\newcommand{\col}{\operatorname{col}}
\newcommand{\diag}{\operatorname{diag}}

\ExplSyntaxOn
\cs_set:Npn \__first_footerline:
  {
    \group_begin:
    \small\sffamily
    \ifnum\theblind>0\relax
    \else
      \__short_authors:
    \fi
    \group_end:
  }
\ExplSyntaxOff

\begin{document}
\let\WriteBookmarks\relax
\let\printorcid\relax

\shorttitle{Andy: A Mathematical Agent for Rigorous Proof and Autonomous Research}
\shortauthors{Zi'an Wang}

\title[mode=title]{Andy: A Mathematical Agent for Rigorous Proof and Autonomous Research}
\title[mode=sub]{Global Exponential Leader-Follower Synchronization of Delayed Heterogeneous Networks via Switching Hybrid Control}

\author[a,b]{Zi'an Wang}
\ead{wza1012@tongji.edu.cn}

\address[a]{School of Mathematical Sciences, Tongji University,
  Shanghai 200092, China}

\address[b]{Key Laboratory of Intelligent Computing and Applications
  (Tongji University), Ministry of Education, Shanghai 200092, China}

\address[]{\normalfont\textbf{Project homepage and source code:}
  \href{https://github.com/mowaiwaim/Andy}{\textcolor{blue}{https://github.com/mowaiwaim/Andy}}}

\begin{abstract}
Andy is an autonomous mathematical research agent that turns a mathematical problem into a traceable proof. It solves or verifies a submitted problem, formulates a literature-grounded new problem through a research-value gate, and carries it through proof construction and final verification.   It organizes proof steps in an executable DAG, verifies each step independently and binds the result to a certificate, retains verified work whose interfaces remain unchanged during local repair, and records the full path from problem formulation to final proof. The system separates proof generation from correctness evaluation and can acquire, retain, retrieve, and reuse knowledge from existing results.  Starting from a self-triggered impulsive consensus result \cite{hong2022consensus}, Andy formulates a global exponential leader-follower synchronization problem for delayed heterogeneous networks with switching communication topologies. The proposed hybrid control combines self-triggered impulses with execution delay and continuous feedback over a recovery window. After each delayed impulse, the feedback cancels the delayed error channel until the pre-impulse history leaves the active delay interval. Sufficient conditions for global exponential synchronization are established, Zeno behavior is excluded for both timing sequences, and a numerical example illustrates the result.
\end{abstract}

\maketitle

\section{Introduction}

Large language models and mathematical reasoning systems can solve increasingly difficult fixed problems. Mathematical research requires a broader workflow. A research system must identify a meaningful question, ground it in relevant literature, construct a long proof, distinguish generation from correctness evaluation, and revise local failures without discarding verified work. These capabilities are often studied separately, so it remains unclear whether a mathematical agent can carry a research task from a verified starting point to a traceable theorem, proof, and numerical certificate.

Related systems provide useful context for this objective.
FunSearch pairs programs proposed by a pretrained large language model with a problem-specific systematic evaluator. AlphaGeometry uses a neural model to generate auxiliary constructions and a symbolic engine to complete geometry proofs \cite{romeraparedes2024mathematical,trinh2024alphageometry}. These systems demonstrate how learned generation can be combined with an explicit external evaluation or deduction mechanism.

Recent mathematical research agents move beyond fixed benchmark problems. Moonshine extracts structure from classical problems, formulates conjectures, builds connections, and identifies obstacles \cite{chen2026moonshine}. Research Math Agents (RMA) targets long-horizon research problems through literature grounding, structured knowledge, iterative proof refinement, and coordinated proposer and verifier roles \cite{zhao2026rma}. QED organizes multiple agents around failure modes observed in open-problem proving and emphasizes structured verification and expert assessment of generated proofs \cite{an2026qed}.

Recent systems distribute the research cycle differently. ProofCouncil uses an author--critic loop implemented by a conditional workflow-DAG library \cite{schmitt2026proofcouncil}. Danus coordinates parallel workers through a verifier-gated fact graph that stores proofs and logical dependencies and supports transitive revocation \cite{liu2026danus}. Rethlas and Archon couple informal research reasoning with Lean formalization \cite{ju2026conjecture}, while MMAT combines natural-language and formal-language provers with a knowledge-base manager in a full-cycle co-pilot architecture \cite{cao2026mechmath}. Table~\ref{tab:agent-capabilities} compares mechanisms explicitly documented in the primary reports.

\begin{table}[pos=!htbp]
\caption{Capabilities reported by representative mathematical research systems.}
\label{tab:agent-capabilities}
\centering
\scriptsize
\setlength{\tabcolsep}{2pt}
\renewcommand{\arraystretch}{1.18}
\begin{tabularx}{.98\linewidth}{@{}>{\raggedright\arraybackslash}p{.13\linewidth}>{\raggedright\arraybackslash}X*{6}{>{\centering\arraybackslash}p{.068\linewidth}}@{}}
\hline
System & Primary research object & \shortstack{New\\problem} & \shortstack{Proof\\decomp.} & \shortstack{Indep.\\check} & \shortstack{Long-term\\memory} & \shortstack{Formal\\proof} & \shortstack{Local\\revocation}\\
\hline
Moonshine \cite{chen2026moonshine} & Conjecture generation and theory exploration & E & NR & P & E & NR & NR\\
RMA \cite{zhao2026rma} & Literature-grounded long-horizon proofs & NR & E & E & P & NR & NR\\
QED \cite{an2026qed} & Natural-language proofs for open problems & NR & E & E & P & NR & P\\
ProofCouncil\newline\cite{schmitt2026proofcouncil} & Author--critic open-problem solving & NR & P & E & P & P & NR\\
Danus \cite{liu2026danus} & Verifier-gated fact-graph proof search & NR & E & E & E & NR & E\\
Rethlas--\newline Archon \cite{ju2026conjecture} & Informal reasoning followed by Lean formalization & NR & E & E & E & E & P\\
MMAT \cite{cao2026mechmath} & Natural- and formal-language research co-pilot & P & E & E & E & E & P\\
Andy & Value-gated problem formulation and incremental proof & E & E & E & E & NR & E\\
\hline
\end{tabularx}

\smallskip
\begin{minipage}{\linewidth}
\scriptsize E denotes an explicitly implemented or reported capability. P denotes partial, case-specific, run-local, or architecturally non-equivalent support. NR records only a lack of reporting in the cited primary source and leaves possible implementations unassessed.
\end{minipage}
\end{table}

Together, these developments motivate mathematical agents whose research objects, proof dependencies, and evaluations remain explicit and auditable.

To address this gap, we developed Andy, a verification-centered autonomous mathematical research agent. Andy combines solver and evaluator separation, literature-grounded problem generation, a research-value gate, DAG-based proof execution, independent verification, targeted repair, and reusable memory in one auditable workflow. The case study in this paper examines whether this workflow can transform a verified control-theoretic result into a technically meaningful new problem and complete the corresponding theorem, proof, and numerical verification.

The main contributions of Andy are summarized as follows.
\begin{itemize}
\item[(1)] We define an executable, version-bound proof DAG. Each node is a five-part research object consisting of a statement, assumptions, dependencies, a proof, and a verification certificate. The certificate is tied to the exact node version that it verifies. This representation converts a proof outline into an executable dependency structure with explicit admissibility and audit conditions.
\item[(2)] We implement a verification-controlled DAG executor. A node becomes available only after all its predecessors have been certified. The executor performs topological execution, freezes certified nodes, compares node interfaces after revision, propagates invalidation only when an interface changes, activates backup routes when a branch fails, and re-verifies the assembled proof. Local repair therefore retains certified work whose dependencies remain valid.
\item[(3)] We provide an end-to-end control-theoretic case study with a complete audit trail. Starting from a verified self-triggered impulsive consensus result, Andy revised the proposed problem eight times across nine versions and explored seven proof routes. It ultimately produced a recovery-window synchronization theorem for delayed heterogeneous networks under switching hybrid control, a complete proof, and a numerical illustration. The recorded trace identifies failed routes, preserved nodes, and the reason for selecting the final route.
\end{itemize}

The remainder of this paper is organized as follows. Section~2 presents the Andy agent architecture. Section~3 gives the control-theoretic case study.  Section~4 concludes the paper.

\section{Andy Agent Architecture}

Andy organizes problem solving, problem generation, proof construction, and verification as a single revisable process. Figure~\ref{fig:andy-workflow} summarizes the complete workflow, and the following subsections describe its main components.

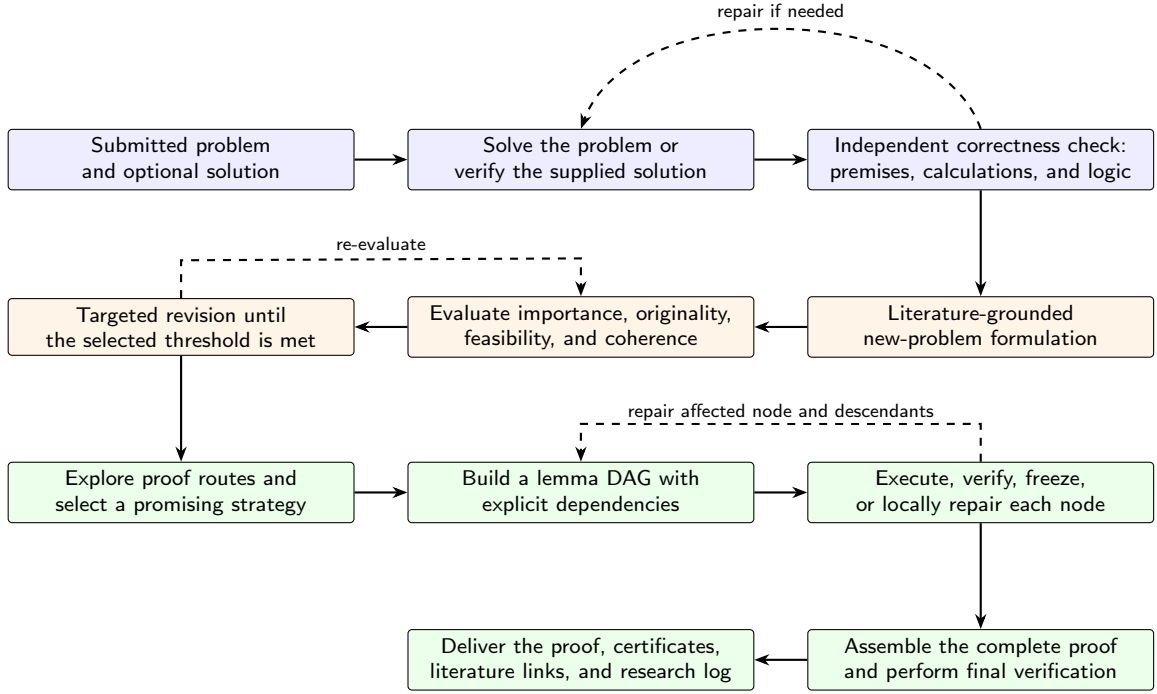
\begin{figure}[pos=!htbp]
	\begin{minipage}{\linewidth}
		\centering
		\begin{tikzpicture}[
			node distance=14mm and 7mm,
			stage/.style={draw, rounded corners=1.5pt, align=center, minimum height=7.5mm,
				text width=0.265\linewidth, font=\footnotesize, inner sep=3pt},
			verify/.style={stage, fill=blue!7},
			research/.style={stage, fill=orange!9},
			proofstage/.style={stage, fill=green!8},
			flow/.style={-{Stealth[length=2mm]}, thick},
			feedback/.style={-{Stealth[length=2mm]}, thick, dashed}
			]
			\node[verify] (input) {Submitted problem\\and optional solution};
			\node[verify, right=of input] (solve) {Solve the problem or\\verify the supplied solution};
			\node[verify, right=of solve] (check) {Independent correctness check:\\premises, calculations, and logic};
			
			\node[research, below=of check] (formulate) {Literature-grounded\\new-problem formulation};
			\node[research, left=of formulate] (evaluate) {Evaluate importance, originality,\\feasibility, and coherence};
			\node[research, left=of evaluate] (revise) {Targeted revision until\\the selected threshold is met};
			
			\node[proofstage, below=of revise] (routes) {Explore proof routes and\\select a promising strategy};
			\node[proofstage, right=of routes] (dag) {Build a lemma DAG with\\explicit dependencies};
			\node[proofstage, right=of dag] (local) {Execute, verify, freeze,\\or locally repair each node};
			
			\node[proofstage, below=of local] (assembly) {Assemble the complete proof\\and perform final verification};
			\node[proofstage, left=of assembly] (certificate) {Deliver the proof, certificates,\\literature links, and research log};
			
			\draw[flow] (input) -- (solve);
			\draw[flow] (solve) -- (check);
			\draw[flow] (check) -- (formulate);
			\draw[flow] (formulate) -- (evaluate);
			\draw[flow] (evaluate) -- (revise);
			\draw[flow] (revise) -- (routes);
			\draw[flow] (routes) -- (dag);
			\draw[flow] (dag) -- (local);
			\draw[flow] (local) -- (assembly);
			\draw[flow] (assembly) -- (certificate);
			\draw[feedback] (check.north) to[out=110,in=70,looseness=0.9]
			node[midway,above,font=\scriptsize]{repair if needed} (solve.north);
			\coordinate (reviseback) at ([yshift=5mm]revise.north);
			\coordinate (evaluateback) at ([yshift=5mm]evaluate.north);
			\draw[feedback] (revise.north) -- (reviseback) --
			node[midway,above,font=\scriptsize]{re-evaluate} (evaluateback) -- (evaluate.north);
			\coordinate (localback) at ([yshift=5mm]local.north);
			\coordinate (dagback) at ([yshift=5mm]dag.north);
			\draw[feedback] (local.north) -- (localback) --
			node[midway,above,font=\scriptsize,fill=white,inner sep=1pt]
			{repair affected node and descendants} (dagback) -- (dag.north);
		\end{tikzpicture}
		\caption{The research workflow of Andy.}
		\label{fig:andy-workflow}
	\end{minipage}
\end{figure}

\subsection{Solver--Evaluator Separation}

Andy uses a dual-model architecture that separates the roles of solver and evaluator. The main solver uses DeepSeek V4 Pro. It is responsible for solving the original problem, generating proofs, generating new problems, and revising proofs. The mathematical correctness evaluator uses Claude Sonnet 5. It reviews answers and proofs along three dimensions: use of premises, correctness of calculations, and logical completeness. The evaluator cannot directly modify a proof. It can only provide structured verification conclusions, error locations, and revision suggestions. The main solver then completes the revisions. This separation of roles prevents the solver model from approving its own answer without external review.

To use Andy, the user must first provide a problem. If no answer is provided, Andy independently constructs a solution or proof and submits it to the correctness evaluator for review. If an answer is provided, Andy directly verifies that answer without solving the problem independently again.

\begin{remark}
During this study, Codex running GPT-5.6 Sol was placed outside Andy and used as a simulated human monitor. It inspected the exposed intermediate reasoning, identified mathematical and strategic errors, and returned targeted guidance. The external monitor supplied the guidance, and Andy's solver retained responsibility for implementing every revision. This setup shows that the current workflow can still benefit from external oversight when deep reasoning errors are difficult for the internal roles to detect.
\end{remark}

\subsection{Literature-grounded Problem Generation}

After the original problem and its proof pass verification, Andy proposes a new research problem. Before a task begins, the user can separately specify the direction for generating the problem or upload reference papers. This guidance is used only during the generation, evaluation, and revision of the new problem. It does not affect the solution and verification of the original problem.

Andy analyzes the common features, differences, complementary relations, and research gaps among the papers. In the new problem, it explains which papers were used, which conditions or methods were inherited, and which substantive modifications were made. When the available materials are insufficient to assess originality, related work, or proof tools, Andy can use Tavily to search for related literature as needed.

\subsection{Research-value Gate}

The basic evaluation of a new problem covers four dimensions: importance, originality, feasibility, and coherence. Importance measures whether solving the problem could advance a theory, method, or application. Originality checks whether the problem only renames, rewrites, or directly restates an existing problem. Feasibility examines whether there is a concrete proof strategy, usable tools, and a reasonable point of entry for the research. Coherence examines whether the new problem follows naturally from the original problem or the references and whether it has one clear research objective. It also checks that the problem is not a simple combination of several small problems. If the user provides reference papers, the evaluation system adds a literature support dimension. This dimension checks whether the new problem has traceable links to the papers and contains substantive modifications.

The user can specify the target difficulty of the new problem. Every dimension is scored on a ten-point scale. An easy problem requires a total score of at least $6.0$ and a score of at least $4.5$ in every dimension. A standard problem requires a total score of at least $7.0$ and a score of at least $5.5$ in every dimension. A research-level problem requires a total score of at least $7.5$ and a score of at least $6.0$ in every dimension. A pure rewriting of an existing problem triggers a hard gate. A simple combination of several conclusions, the lack of a concrete feasible method, or the inclusion of several parallel subproblems also triggers a hard gate. Such a problem does not pass verification. The user's choice of difficulty changes only the threshold for research value and the research cost. It does not reduce the correctness requirements for the final proof.

If a problem does not pass the evaluation, the evaluation stage explains the reason for each score deduction. It also gives revision instructions that include the specific defect, the mathematical modification, and the acceptance criterion after modification. The generator then revises the same problem draft in a targeted manner and returns it to the evaluation stage for another review. By default, the system automatically enters the proof stage after the problem reaches the threshold. The user can also choose to pause after the evaluation and intervene manually. Algorithm~\ref{alg:problem-gating} summarizes this formulation, evaluation, and targeted-revision loop.

\begin{algorithm}[!htbp]
\caption{Problem formulation and quality gating}
\label{alg:problem-gating}
\begin{algorithmic}[1]
\Require Verified starting result \(R\), literature \(L\), user direction \(U\), and target thresholds \(\theta\)
\Ensure Accepted problem \(Q^\star\) and audit record \(\mathcal A\)
\State \(Q_0\gets\Call{Formulate}{R,L,U}\) and initialize \(\mathcal A\)
\For{\(j=0,1,\ldots\)}
  \State \((\mathbf s_j,\mathbf g_j)\gets\Call{Evaluate}{Q_j,L,\theta}\)
  \State Append \(Q_j\), dimension scores \(\mathbf s_j\), and gate verdicts \(\mathbf g_j\) to \(\mathcal A\)
  \If{all score thresholds and hard gates pass}
    \State \(Q^\star\gets Q_j\); \textbf{return} \((Q^\star,\mathcal A)\)
  \EndIf
  \State \(\Delta_j\gets\Call{TargetedRevision}{Q_j,\mathbf s_j,\mathbf g_j}\)
  \State \(Q_{j+1}\gets\Call{Revise}{Q_j,\Delta_j}\)
\EndFor
\end{algorithmic}
\end{algorithm}

\subsection{DAG Proof Execution}

For the proof of a new problem, Andy provides two routes: a rigorous DAG proof route and a deep research route. A directed acyclic graph (DAG) is a directed graph with no directed cycles and therefore admits a topological ordering. In artificial intelligence, DAGs are widely used to encode conditional-dependence structures and organize probabilistic inference \cite{pearl1988probabilistic}.

Andy adapts this dependency structure to proof construction. The rigorous DAG proof route first generates only a proof outline, necessary definitions, and lemmas. Andy then represents the proof as a directed acyclic graph. Each node corresponds to a lemma or proof step with explicit inputs, assumptions, and a conclusion. Directed edges represent logical dependencies between nodes. The system first checks node identifiers, dependency relations, and reachability of the target. It then executes the nodes one by one in topological order.

DAG mode starts a main proof branch. If the main branch cannot be completed mathematically, Andy activates an alternative branch. If the new DAG executor and the alternative branch both fail, the system can also use a multibranch proof chain as a fallback and try several proof routes.

For a research-level problem, the deep research route can be selected. This route first explores several proof strategies at low cost. It compares only the key construction, required lemmas, main risks, and likelihood of completion for each route. After selecting the most promising route, the system converts it into a DAG. It then completes the proof rigorously through node-by-node execution, local verification, freezing, and final overall verification. For a difficult and open-ended problem, this route first obtains the necessary breadth of exploration at low cost. It then concentrates most of the computational effort on the most promising route.

For execution, a node version is recorded as \(v=(S_v,A_v,D_v,\Pi_v,C_v)\), where \(S_v\), \(A_v\), \(D_v\), \(\Pi_v\), and \(C_v\) are its statement, assumptions, dependency set, proof, and verification certificate, respectively. Its dependency interface is \(I_v=(S_v,A_v,D_v)\), and \(h_v=H(I_v)\) is the corresponding interface fingerprint. The certificate is bound to the exact node version and its execution metadata. Algorithm~\ref{alg:dag-execution} gives the corresponding topological execution and verification procedure.

\begin{algorithm}[!htbp]
\caption{Proof-DAG execution}
\label{alg:dag-execution}
\begin{algorithmic}[1]
\Require Accepted problem \(Q^\star\) and planned main and backup routes
\Ensure A frozen target node or a certified failure record
\State Build a versioned DAG \(G\) and reject cycles, missing dependencies, or an unreachable target
\While{the target node is not frozen}
  \State Select a ready node \(v\) in topological order; every node in \(D_v\) must be frozen
  \State \(\Pi_v\gets\Call{Execute}{S_v,A_v,D_v}\)
  \State \(C_v\gets\Call{IndependentVerify}{S_v,A_v,D_v,\Pi_v}\)
  \If{\(C_v\) passes}
    \State Bind \(C_v\) to the node version and execution metadata; \(\Call{Freeze}{v}\)
  \Else
    \State Repair \(v\), recompute \(h_v\), and apply Algorithm~\ref{alg:invalidation-assembly}
  \EndIf
  \If{the active route is mathematically infeasible}
    \State Activate the highest-ranked unused backup route
  \EndIf
\EndWhile
\State Assemble the proof from frozen nodes and perform an independent final verification
\end{algorithmic}
\end{algorithm}

\begin{remark}
	The present evaluation is preliminary and focuses on the efficiency of DAG execution in one representative task. The main purpose of the DAG is to freeze the parts that have been confirmed as correct and confine errors to local regions. This reduces repeated generation and repeated verification of the entire proof.
	
	We used Andy twice to process the problem ``Study the zeros and monotonicity of the function $x^3-x$.'' An ordinary proof attempt and verification process serves as the baseline. It took $1$ hour, $22$ minutes, and $12$ seconds. The DAG process took approximately $37$ minutes and $59.7$ seconds. The time decreased by approximately $53.8\%$.
	
	The new research problems generated by Andy in the two runs were not identical because of the randomness of the LLM. The experiment shows that the DAG avoided unnecessary execution of complete branches and substantially reduced the running time in this case. The entire time difference cannot be attributed to the DAG.
	
\end{remark}

\subsection{Verification and Targeted Repair}

When a node is executed, the solver can use only the available materials and dependency nodes that have passed verification. After each new node is completed, it undergoes a separate local verification of premises, calculations, and logic. A node is frozen after it passes verification. If a node fails, Andy repairs only that node and the descendants affected by it. Unrelated sibling nodes are not regenerated. If a revision changes only the proof text and leaves the proposition, assumptions, and dependency interface of the node unchanged, verified downstream results can be retained. After all necessary nodes pass, the system assembles the complete proof and performs a unified final verification. If the assembly stage requires a new bridging result, that result must first become a new node and pass a separate verification. Let \(\operatorname{Desc}_G(v)\) denote the direct and transitive descendants of node \(v\) in \(G\). Algorithm~\ref{alg:invalidation-assembly} makes the repair and assembly contract explicit.

\begin{algorithm}[!htbp]
\caption{Invalidation and final assembly}
\label{alg:invalidation-assembly}
\begin{algorithmic}[1]
\Require Proof DAG \(G\), revised node \(v\), and old and new interface fingerprints \(h_v^{\mathrm{old}},h_v^{\mathrm{new}}\)
\Ensure Updated valid DAG and a final verification record
\State \(\mathcal I(v)\gets
\begin{cases}
\{v\}, & h_v^{\mathrm{old}}=h_v^{\mathrm{new}},\\
\{v\}\cup\operatorname{Desc}_G(v), & h_v^{\mathrm{old}}\ne h_v^{\mathrm{new}}
\end{cases}\)
\ForAll{\(w\in\mathcal I(v)\)}
  \State Revoke the frozen state and certificate of \(w\)
\EndFor
\State Re-execute affected nodes in topological order under Algorithm~\ref{alg:dag-execution}
\State Assemble the final proof using frozen nodes only
\If{assembly introduces a new bridging proposition}
  \State Create a new node for that proposition and verify it before continuing
\EndIf
\State Independently verify the assembled proof and bind the final record to all cited node versions
\end{algorithmic}
\end{algorithm}

\FloatBarrier

The executor maintains four system invariants.
\begin{itemize}
	\item[(1)] Dependencies must consist exclusively of verified nodes.
	\item[(2)] Each certificate is strictly bound to a specific statement, proof, context, model, prompt, and code version.
	\item[(3)] Any bridging proposition introduced during assembly must be represented and verified as a new node.
	\item[(4)] The final proof must not cite any invalidated node.
\end{itemize}

\subsection{Memory and Knowledge Acquisition}

Andy is designed for long-term use and debugging by researchers. It therefore also has a memory function. Andy manages logs and long-term memory separately. The logs provide complete records of problem versions, model rounds, reasoning content returned by model interfaces, tool calls, search results, evaluation reasons, verification certificates, DAG node revisions, frozen and invalidated states, causes of errors, and task checkpoints. These materials are saved as session records and written to the project-level research log. A failed or interrupted task can therefore resume from the corresponding stage.

Long-term memory contains selected reusable research information. Andy can store proof-method preferences and evaluation preferences explicitly stated by the researcher. Examples include ``Prefer the Lyapunov method for problems of this type in the future'' and ``A rewriting that only changes notation should not receive a high originality score.'' It also stores verified conclusions, successful proof patterns, failed routes, key difficulties, and connections among papers. A new task retrieves only a small amount of memory relevant to the current problem. Failure records help avoid repeating unsuccessful routes and are stored separately from mathematical evidence.

Andy organizes its data around a runtime home directory. This directory contains configurations and credentials, project workspaces, session logs, a session database, a knowledge base, long-term memory, Agent definitions, Skills, Tools, runtime logs, and MCP server description files. During a research task, Andy can independently select memory retrieval, knowledge-base queries, literature searches, and verification tools according to the current stage. It updates the research state through a cycle of planning, execution, verification, and local revision.

We also designed a research front end to make the system easier to use and review. The front end integrates the creation of a research task, a real-time research workspace, and the final report into a complete workflow. Users can upload materials for the main problem and reference papers. They can separately enter the research requirements for the original problem and the direction for generating a new problem. They can also select the target difficulty and proof mode. During execution, users can view stable stage progress, total running time, the actual \texttt{reasoning\_content} returned by the model interface, the verification process, evaluation reasons, literature connections, and the dependencies, revisions, frozen states, and certificate states of DAG nodes. If users find an error in the model's reasoning, they can append a prompt to correct the research direction. This provides an interface for human intervention. After the task is complete, users can directly download research outputs at different levels from the front end.

Andy can ultimately deliver the original problem and its verification conclusion or complete solution, the new problem generated by Andy, itemized scores and revision records, the complete proof of the new problem, and connections to reference papers and retrieved literature. If the proof fails, it returns the reason for failure and subsequent reduced-scope problems. The system can export Markdown, structured JSON, and both complete and concise PDF reports. The complete report retains scores, tools, certificates, versions, and the research process for review and auditing. The concise report retains only the original problem and its complete solution, together with the problem generated by Andy and its complete proof. It is suitable for direct reading and communication.

\FloatBarrier

\section{Case Study}

\subsection{Source Problem}

The verified source problem is the self-triggered impulsive consensus problem of Hong and Zhang \cite{hong2022consensus}. It studies switched delay multi-agent systems under self-triggered impulsive control. Andy uses this result as a starting point and asks how its scheduling and impulsive-control structure can be extended to a heterogeneous leader-follower network with switching topology and impulse execution delay.

\subsection{Problem Proposed by Andy}

We made only minor formatting adjustments to the model and theorem proposed by Andy. The result is presented below.

This case study considers a class of leader-follower heterogeneous networks with switching communication topologies, time-varying delays, and impulse execution delays. The aim is to establish sufficient conditions for synchronization under the joint action of continuous and impulsive control. The model and its regularity assumptions are introduced first.

Let $N$ and $n$ be positive integers. Throughout this section, $I_m$ denotes the $m\times m$ identity matrix, $\otimes$ denotes the Kronecker product, and $\|\cdot\|$ denotes the Euclidean vector norm or its induced matrix norm. For symmetric matrices, $X\preceq Y$ means that $Y-X$ is positive semidefinite.

Let $\{s_k\}_{k\ge0}$ be the sampling instants determined by the self-triggered algorithm introduced below. Let $\tau>0$ be the known impulse execution delay, and define the corresponding impulse execution instants by $t_k:=s_k+\tau$. All state trajectories are taken to be right-continuous at the impulse execution instants.

Let $\bar h>0$, let $h:[s_0,\infty)\to\R$ be a time-varying delay, and let $f,g:\R^n\to\R^n$ be nonlinear functions.

Consider the following $n$-dimensional leader:
\begin{equation}
  \dot s(t)=-Cs(t)+Af(s(t))+Bg(s(t-h(t))),
  \label{eq:leader}
\end{equation}
where $C,A,B\in\R^{n\times n}$.

The $i$th follower, $i=1,\dots,N$, satisfies the hybrid dynamics
\begin{equation}
	\left\{
	\begin{aligned}
		\dot x_i(t)
		&=-C_i x_i(t)+A_i f(x_i(t))
		+B_i g\bigl(x_i(t-h(t))\bigr)+u_i^c(t),
		&& t\neq t_k,\\
		\Delta x_i(t_k)
		&:=x_i(t_k^+)-x_i(t_k^-)=u_i^p(t_k),
		&& t=t_k,
	\end{aligned}
	\right.
	\label{eq:follower}
\end{equation}
where $C_i,A_i,B_i\in\R^{n\times n}$ may differ among the nodes. The terms
$u_i^c(t)$ and $u_i^p(t_k)$ denote the continuous and impulsive control inputs, respectively.
The communication topology among the network nodes can switch and will be introduced below.
\begin{assumption}\label{ass:regularity}
	The time-varying delay $h$ is absolutely continuous and satisfies
$
		0\le h(t)\le\bar h,
		\dot h(t)\le\delta<1
		\,\text{a.e.}
		$
	
	The nonlinear functions $f$ and $g$ are globally Lipschitz continuous. Namely, there exist $l_f,l_g>0$ such that, for any $x,y\in\R^n$,
	\begin{equation}
		\|f(x)-f(y)\|\le l_f\|x-y\|,
		\qquad
		\|g(x)-g(y)\|\le l_g\|x-y\|.
		\label{eq:lipschitz}
	\end{equation}

\end{assumption}

\subsection{Key Idea: Recovery Window}

Andy adopts the following hybrid control idea. The network error is sampled, and impulsive control is applied after a prescribed execution delay. A recovery window whose length equals the upper bound on the time-varying state delay is then introduced. During this window, the continuous controller temporarily cancels the delay channel in the error system. When the window ends, the pre-impulse history has left the active delay interval. The delayed error channel is then reactivated, and the stability analysis resumes with the complete Lyapunov-Krasovskii functional.

The continuous controller is designed as
\begin{align}
	u_i^c(t)={}&(C_i-C)s(t)+(A-A_i)f(s(t))
	+(B-B_i)g(s(t-h(t)))\nonumber\\
	&-K_i\bigl(x_i(t)-s(t)\bigr)
	-(1-q(t))B_i\bigl[g(x_i(t-h(t)))-g(s(t-h(t)))\bigr],
	\label{eq:continuous-controller}
\end{align}
where $K_i\in\R^{n\times n}$ is the feedback gain. The phase indicator $q(t)\in\{0,1\}$ is right-continuous, distinguishes the recovery phase from the normal-flow phase, and is coordinated with the impulse instants. Specifically, (i) choose the initial sampling instant $s_0$ such that $q(s_0^-)=1$, and set $q(t)=1$ on the initial interval $[s_0,t_0)$. (ii) Set $q(t_k)=q(t_k^+)=0$ whenever $t=t_k$. (iii) Keep $q=0$ in the recovery window $[t_k,t_k+\bar h)$, and set $q=1$ on $[t_k+\bar h,t_{k+1})$. (iv) A topology switch that does not coincide with an impulse does not reset $q$.

We next introduce impulsive control with a switching communication topology. First, consider the switching graphs. Let $l$ be a positive integer, and let $\sigma(t)\in\mathcal S:=\{1,\ldots,l\}$ be a left-continuous, piecewise constant switching signal. In mode $r\in\mathcal S$, the communication graph $\mathcal G_r$ is a connected undirected graph. Its Laplacian matrix is denoted by
$
L_r=(\ell_{ij}^r)_{N\times N}.
$

Let $a_{ij}^r$ denote the communication weight between nodes $i$ and $j$ in mode $r$. Then
\[
\ell_{ij}^r=
\begin{cases}
	-a_{ij}^r, & i\neq j,\\[1mm]
	\displaystyle\sum_{m\neq i}a_{im}^r, & i=j.
\end{cases}
\]

Let $N_\sigma(t,u)$ denote the number of topology switches in the interval $[u,t)$, and assume that
\begin{equation}
	N_\sigma(t,u)\le \frac{t-u}{T_a}+N_0,
	\qquad t\ge u\ge s_0,
	\label{eq:adt}
\end{equation}
where $T_a>0$ is the average dwell time and $N_0\ge0$ is the chatter bound.

\begin{assumption}

The mode schedule satisfies
\begin{equation}
	\sigma(t)=r_k:=\sigma(s_k^-),
	\qquad
	s_k\le t\le t_k+\bar h,
	\qquad k\ge0.
	\label{eq:mode-hold}
\end{equation}
It also satisfies
\begin{equation}
	\sigma\bigl((t_k+\bar h)^+\bigr)=r_k,
	\qquad k\ge0.
	\label{eq:mode-hold-right}
\end{equation}
Thus, the topology mode remains $r_k$ from each sampling instant $s_k$ to the end of the corresponding recovery window. Its right limit at the endpoint of the recovery window is also $r_k$. In particular, no topology switching is allowed in $[s_k,t_k+\bar h]$, including at $t_k$. An immediate switch to the right of the endpoint of the recovery window is also excluded.
	
\end{assumption}

Choose the pinning matrix
$
D=\diag(d_1,\ldots,d_N), d_i\ge0,
$
where $d_i>0$ means that the $i$th follower has an additional leader-pinning channel in the impulsive coupling. When $d_i=0$, the node receives the leader's influence only indirectly through its neighbors in the impulsive topology. Given the impulse gain $\mu_k>0$, define the topology-dependent impulse mapping matrix
$
	M_k:=\mu_k\bigl(L_{\sigma(s_k^-)}+D\bigr),
$
where the active topology immediately before the sampling instant $s_k$ is used. The mode-scheduling assumption in \eqref{eq:mode-hold} and \eqref{eq:mode-hold-right} avoids a mismatch between the sampling topology and the impulse execution topology.

\begin{remark}
	The impulsive control law selects the corresponding Laplacian matrix according to the active communication topology at the sampling instant. It is therefore a topology-dependent switching impulsive controller. This design can describe practical situations such as changes in adjacency relations caused by moving nodes, failures and recovery of communication links, and time-sharing schedules for wireless channels. The controller can then implement impulsive corrections according to the communication relations available at the sampling instant.
\end{remark}

Define the synchronization errors by
$
	e_i(t):=x_i(t)-s(t),
	e(t):=\col(e_1(t),\ldots,e_N(t)).
$

Assume that the leader and follower state histories are continuous on $[s_0-\bar h,s_0]$. Define the initial error segment by
\[
	e_{s_0}(\theta):=e(s_0+\theta),
	\qquad \theta\in[-\bar h,0],
\]
so $e_{s_0}\in\mathcal C([-\bar h,0],\R^{Nn})$, where $\mathcal C$ denotes the space of continuous functions. Define its norm by
\[
	\|e_{s_0}\|_{\bar h}
	:=\sup_{-\bar h\le\theta\le0}
	\|e(s_0+\theta)\|.
\]

Let $[z]_i$ denote the $i$th $n$-dimensional block of the stacked vector $z\in\R^{Nn}$. Since the leader state $s(t)$ is continuous at the impulse execution instants, the impulsive control input is designed as
\begin{align}
	u_i^p(t_k)
	&:=\left[(M_k\otimes I_n)e(s_k^-)\right]_i-e_i(t_k^-)
	\label{eq:impulsive-controller}\\
	&=-e_i(t_k^-)
	+\mu_k\left[
	\sum_{j=1}^N a_{ij}^{\sigma(s_k^-)}
	\bigl(e_i(s_k^-)-e_j(s_k^-)\bigr)
	+d_i e_i(s_k^-)
	\right].
	\nonumber
\end{align}

Since $x_i(t_k)=x_i(t_k^-)+u_i^p(t_k)$, the stacked error after impulse execution is
\begin{equation}
	e(t_k)
	=(M_k\otimes I_n)e(s_k^-).
	\label{eq:impulse-map}
\end{equation}

Substituting \eqref{eq:continuous-controller} into the leader and follower systems in \eqref{eq:leader} and \eqref{eq:follower} gives the unified closed-loop error system
\begin{equation}
	\left\{
	\begin{aligned}
		\dot e_i(t)
		&=-(C_i+K_i)e_i(t)+A_i\Delta f_i(t)
		+q(t)B_i\Delta g_i(t-h(t)),
		&& t\neq t_k,\\
		e_i(t_k)
		&=\left[(M_k\otimes I_n)e(s_k^-)\right]_i,
		&& t=t_k,
	\end{aligned}
	\right.
	\label{eq:error-system}
\end{equation}
where
$
\Delta f_i(t):=f(x_i(t))-f(s(t)),\,
\Delta g_i(t-h(t))
:=g(x_i(t-h(t)))-g(s(t-h(t))).
$

It follows from \eqref{eq:error-system} that the delay channel in the error system is exactly canceled when $q(t)=0$. This delay channel is reactivated when $q(t)=1$. Furthermore, suppose that the error is identically zero at a given instant and throughout the required delay-history interval. Then $\Delta f_i(t)=0$ and $\Delta g_i(t-h(t))=0$, so the continuous flow satisfies $\dot e_i(t)=0$. If $e(s_k^-)=0$, the impulse mapping gives
$
e(t_k)=(M_k\otimes I_n)e(s_k^-)=0.
$
Therefore, the continuous flow and the impulsive jumps both preserve the zero-error state. Once the system reaches synchronization, it remains synchronized.

We next introduce self-triggered impulsive control. Let the triggering parameters be
$\bar\lambda\ge0$ and $\rho>0$, and the parameter sequence be
$\{a_k\}_{k\ge0}$, where $a_k>0$. Assume additionally that $\tau\ge\bar h$.

The selected triggering parameter $\bar\lambda$ will be further constrained in the theorem below by the Lyapunov matrices and their related constants.

At the computed sampling instant $s_k$, set $W_k^2:=e^{\bar\lambda\tau}$ and define the self-triggering function for a candidate instant $t\ge t_k$ as
\begin{equation}
	\Psi_k(t)
	:=(\bar\lambda+\rho)(t-t_k)-a_k-\ln W_k^2.
	\label{eq:trigger-function}
\end{equation}

Since $\bar\lambda+\rho>0$, the function $\Psi_k(t)$ is strictly increasing in $t$. Therefore, the next sampling instant is defined as the first instant at which $\Psi_k(t)$ reaches a nonnegative value. Namely,
\begin{equation}
	s_{k+1}
	:=\inf\bigl\{t>t_k:\Psi_k(t)\ge0\bigr\}.
	\label{eq:self-trigger}
\end{equation}

Define the time length from the $k$th actual impulse execution instant $t_k$ to the next sampling instant $s_{k+1}$ by
\begin{equation}
	\begin{aligned}
		\Phi_k
		:=s_{k+1}-t_k
		=\frac{a_k+\ln W_k^2}{\bar\lambda+\rho}
		=\frac{a_k+\bar\lambda\tau}{\bar\lambda+\rho}.
	\end{aligned}
	\label{eq:Phi}
\end{equation}

Hence,
$
	s_{k+1}=t_k+\Phi_k
	=s_k+\tau+\Phi_k,
$

Therefore, the sampling sequence and the impulse execution sequence are interlaced as
$
	s_0<t_0<s_1<t_1<\cdots,
$
and
$
	t_{k+1}-t_k
	=s_{k+1}-s_k
	=\tau+\Phi_k>0.
$

Both the sampling intervals and the impulse execution intervals have the uniform strict positive lower bound $\tau$. Thus, any finite time interval contains only finitely many sampling instants and impulse execution instants. Therefore, neither the self-triggered sampling sequence nor the impulse execution sequence exhibits Zeno behavior.

\begin{remark}
	The recovery-window feedback in \eqref{eq:continuous-controller} reveals a useful control-design idea. It matches a controller term with a difficult term in the error dynamics and removes that term during the part of the hybrid evolution in which it is most troublesome.

	Andy's construction complements standard delay-analysis methods with a structural heuristic. It first identifies the term that makes the Lyapunov derivative difficult. The controller and the Lyapunov function are then co-designed so that this term is canceled or absorbed in a controlled phase. This principle also appears in PD-controlled multi-weighted networks. A derivative feedback term can be matched with a coupling-weighted quadratic term in the Lyapunov function so that the corresponding cross-derivative terms cancel pairwise \cite{wang2021pdpi,wang2024finite}. In PI designs, related weight terms can be balanced in an augmented Lyapunov functional \cite{wang2021pdpi}.

	A broader version of this idea appears in recursive adaptive control. Coordinate changes, parameter-update laws, and the feedback law are constructed step by step, and stability is established through a Lyapunov argument \cite{kanellakopoulos1991systematic}. A related learning-based approach jointly trains a nonlinear controller and a neural Lyapunov function. An SMT-based falsification step checks the Lyapunov conditions and returns counterexamples when they fail \cite{zhou2022neural}.

	These connections suggest that Andy's control design can inform the construction of Lyapunov functions and PID-type controllers for systems with several coupling weights, delayed channels, or derivative couplings.

	Exact delay-channel cancellation in \eqref{eq:continuous-controller} is demanding in practice because it requires accurate knowledge of the delayed model and online access to the complete delayed states. Mature delay-system methods can retain the delayed channel and estimate its effect through Lyapunov-Krasovskii functionals, comparison arguments, and delayed impulsive inequalities \cite{hong2022consensus,zhang2023stabilizing,xie2024selftriggered}. When exact delayed states are unavailable, a natural extension is to use observer-based or adaptive approximate cancellation and bound the residual mismatch with standard delay-analysis tools.
\end{remark}

\subsection{Main Theorem}

For any candidate symmetric positive definite matrices $P_r,S_r,R_r\in\R^{n\times n}$, constants $\lambda_{1,r}\ge0$, and scalars $\varepsilon_{i,r}>0$ appearing in condition (i) below, where $r\in\mathcal S$ and $i=1,\ldots,N$, define
\begin{align*}
	Q_{i,r}:={}&
	P_r(C_i+K_i)+(C_i+K_i)^{\mathsf T}P_r,\\
	\Lambda_{i,r}^{(11)}:={}&
	-Q_{i,r}+2\|P_rA_i\|l_fI_n
	+\varepsilon_{i,r}I_n
	+\lambda_{\max}(S_r)I_n
	+\bar h\lambda_{\max}(R_r)I_n,\\
	\Lambda_{i,r}^{(22)}:={}&
	\varepsilon_{i,r}^{-1}\|P_rB_i\|^2l_g^2I_n
	-(1-\delta)\lambda_{\min}(S_r)I_n.
\end{align*}

Further define
\begin{align*}
	\beta_r^{(0)}
	&:=\max_{1\le i\le N}
	\bigl\{-\lambda_{\min}(Q_{i,r})
	+2\|P_rA_i\|l_f\bigr\}, \qquad
	\beta_r
	:=\frac{\max\{0,\beta_r^{(0)}\}}
	{\lambda_{\min}(P_r)},\\
	\gamma_r^{(0)}
	&:=\max_{1\le i\le N}
	\bigl\{\lambda_{\max}(Q_{i,r})
	+2\|P_rA_i\|l_f\bigr\}, \qquad
	\gamma_r
	:=\frac{\max\{0,\gamma_r^{(0)}\}}
	{\lambda_{\min}(P_r)}.
\end{align*}

Let
\begin{align*}
	\widetilde\chi_r
	&:=e^{(\beta_r+\gamma_r)\bar h}
	\left(
	1+
	\frac{
		\bar h\lambda_{\max}(S_r)
		+\frac{\bar h^2}{2}\lambda_{\max}(R_r)
	}{
		\lambda_{\min}(P_r)
	}
	\right),
	\\
	\widetilde\chi_{\max}
	&:=\max_{r\in\mathcal S}\widetilde\chi_r,
	\qquad
	\beta_+:=\max_{r\in\mathcal S}\beta_r,\qquad
	\bar\lambda_1
	:=\max_{r\in\mathcal S}\lambda_{1,r}.
\end{align*}

For a candidate sequence $\{\eta_k\}_{k\ge0}\subset(0,1)$ appearing in condition (ii) below, define the effective contraction over the $k$th impulsive period by
\begin{equation}
	\widetilde c_k
	:=-\ln\eta_k-\beta_+\bar h
	-\ln\widetilde\chi_{\max}.
	\label{eq:ck}
\end{equation}

For any $\alpha>0$, define
$
	\alpha_{\rm eff}:=\alpha+\bar\lambda-\bar\lambda_1.
$

\begin{theorem}\label{thm:main}
Consider the systems in \eqref{eq:leader} and \eqref{eq:follower} under the continuous controller \eqref{eq:continuous-controller}, the impulsive controller \eqref{eq:impulsive-controller}, and the self-triggering mechanism defined by \eqref{eq:trigger-function} through \eqref{eq:Phi}. Assumption \ref{ass:regularity}, the average dwell-time constraint \eqref{eq:adt}, the execution-delay requirement $\tau\ge\bar h$, and the mode-scheduling assumption in \eqref{eq:mode-hold} and \eqref{eq:mode-hold-right} are imposed throughout.
	If the following conditions hold:
	
	\begin{enumerate}[label=\textup{(\roman*)},leftmargin=3em]
		\item
		For each switching mode $r\in\mathcal S$, there exist symmetric positive definite matrices
		$
		P_r,S_r,R_r\in\R^{n\times n},
		$
		and a constant $\lambda_{1,r}\ge0$ that is valid for all nodes. For every
		$i=1,\ldots,N$, there exists $\varepsilon_{i,r}>0$ such that
		\begin{equation}
			\begin{bmatrix}
				\Lambda_{i,r}^{(11)}&0\\
				0&\Lambda_{i,r}^{(22)}
			\end{bmatrix}
			\preceq
			\lambda_{1,r}
			\begin{bmatrix}
				P_r&0\\
				0&0
			\end{bmatrix},
			\label{eq:flow-lmi}
		\end{equation}
		In addition,
		$
			\bar\lambda\ge\max\{\bar\lambda_1,\beta_+\}.
		$
		
		\item
		
		There exists a sequence $\{\eta_k\}_{k\ge0}\subset(0,1)$ such that, for every $k\ge0$, the impulsive map satisfies
		\begin{equation}
			M_k^{\mathsf T}M_k\preceq\eta_kI_N.
			\label{eq:impulse-contraction-condition}
		\end{equation}
		In addition, there exists a constant $\mu\ge1$ such that, for all
		$r,j\in\mathcal S$,
		\begin{equation}
			P_r\preceq\mu P_j,\qquad
			S_r\preceq\mu S_j,\qquad
			R_r\preceq\mu R_j.
			\label{eq:mode-comparison}
		\end{equation}
		
		\item
		The self-triggering parameters satisfy
		\begin{equation}
			a_k>\rho\bar h,
			\qquad k\ge0.
			\label{eq:trigger-margin}
		\end{equation}
		There exist constants $\alpha>0$ and $\varphi>0$ such that, for all
		$t\ge u\ge s_0$,
		\begin{equation}
			\begin{aligned}
				&\bar\lambda(t-u)
				+\frac{\ln\mu}{T_a}(t-u)
				+\sum_{k:\,u\le t_k<t}
				\bigl[
				-\widetilde c_k
				+(\bar\lambda+\rho)\tau
				\bigr]
				\le-\alpha(t-u)+\varphi.
			\end{aligned}
			\label{eq:cumulative-decay}
		\end{equation}
	\end{enumerate}
	
	Then $\alpha_{\rm eff}\ge\alpha>0$, and the closed-loop system achieves global exponential leader-follower synchronization. Specifically, there exists a constant $C_{\rm GE}\ge1$, independent of the initial history, such that every admissible initial history satisfies
	\begin{equation}
		\|e(t)\|
		\le
		C_{\rm GE}\|e_{s_0}\|_{\bar h}
		e^{-\frac{\alpha_{\rm eff}}{2}(t-s_0)},
		\qquad t\ge s_0,
		\label{eq:global-exp}
	\end{equation}

\end{theorem}
\begin{proof}
The stated regularity and non-Zeno timing assumptions ensure that the closed-loop system admits a unique global right-continuous solution. It remains to establish the exponential estimate.

Fix matrices $P_r,S_r,R_r$ and constants $\lambda_{1,r},\varepsilon_{i,r}$ satisfying condition (i). In each recovery window, construct the quadratic Lyapunov function
\begin{equation}
	\mathcal V_r^0(t):=e^{\mathsf T}(t)(I_N\otimes P_r)e(t),
	\label{eq:V0}
\end{equation}

In each normal phase, construct the Lyapunov-Krasovskii functional
\begin{align}
	\mathcal V_r^1(t):={}&e^{\mathsf T}(t)(I_N\otimes P_r)e(t)+\int_{t-h(t)}^t e^{\mathsf T}(v)(I_N\otimes S_r)e(v)\,\dd v+\int_{-\bar h}^{0}\int_{t+\theta}^{t}
	e^{\mathsf T}(v)(I_N\otimes R_r)e(v)\,\dd v\,\dd\theta.
	\label{eq:Vfull}
\end{align}

In the following, write $\mathscr V(t)=\mathcal V_{\sigma(t)}^0(t)$ when $q(t)=0$, and write $\mathscr V(t)=\mathcal V_{\sigma(t)}^1(t)$ when $q(t)=1$. The initialization of $q$ gives $q=1$ at $s_0^-$ and throughout $[s_0,t_0)$. For $k\ge1$, the conditions $\tau\ge\bar h$ and \eqref{eq:trigger-margin} imply that $q=1$ at each sampling instant $s_k$ and each pre-impulse instant $t_k^-$. Thus, $\mathscr V$ is the full functional at these instants.

First, consider the continuous flow in a recovery window. When $q=0$ and the mode is fixed, \eqref{eq:error-system} has no delay term. Its error trajectory is absolutely continuous. Hence, $\mathcal V_r^0$ is also absolutely continuous. It follows from \eqref{eq:V0}, the definition of $Q_{i,r}$, and \eqref{eq:lipschitz} that, for almost every $t$,
\begin{align}
\dot{\mathcal V}_r^0(t)
={}&\sum_{i=1}^N\bigl[-e_i^{\mathsf T}Q_{i,r}e_i
 +2e_i^{\mathsf T}P_rA_i\Delta f_i\bigr]
\le{}\beta_r^{(0)}\sum_{i=1}^N\|e_i\|^2
\le\beta_r\mathcal V_r^0(t).
\label{eq:V0-upper}
\end{align}

Similarly,
\begin{align}
\dot{\mathcal V}_r^0(t)
&\ge-\gamma_r^{(0)}\sum_{i=1}^N\|e_i\|^2
\ge-\gamma_r\mathcal V_r^0(t).
\label{eq:V0-lower}
\end{align}

Therefore, in the almost-everywhere sense,
$-\gamma_r\mathcal V_r^0(t)\le\dot{\mathcal V}_r^0(t)
\le\beta_r\mathcal V_r^0(t)$. Apply the standard Gronwall inequality to these two differential inequalities. In a recovery window with a fixed mode, for any
$t_k\le u\le v\le t_k+\bar h$, this gives
\begin{equation}
  \mathcal V_r^0(v)\le e^{\beta_r(v-u)}\mathcal V_r^0(u),
  \qquad
  \mathcal V_r^0(u)\le e^{\gamma_r(v-u)}\mathcal V_r^0(v).
  \label{eq:two-sided-gronwall}
\end{equation}

Next, consider the continuous flow in a normal phase. Let $e_{i,h}:=e_i(t-h(t))$. When $q=1$, take the Dini derivative of the quadratic term in \eqref{eq:Vfull} along \eqref{eq:error-system}. We obtain
\begin{align*}
&D^+\!\left[e^{\mathsf T}(t)(I_N\otimes P_r)e(t)\right]
=\sum_{i=1}^N\bigl[
-e_i^{\mathsf T}Q_{i,r}e_i
+2e_i^{\mathsf T}P_rA_i\Delta f_i
+2e_i^{\mathsf T}P_rB_i\Delta g_i(t-h(t))
\bigr].
\end{align*}

The Lipschitz conditions and Young's inequality give, respectively,
\begin{align*}
&2e_i^{\mathsf T}P_rA_i\Delta f_i
\le 2\|P_rA_i\|l_f\|e_i\|^2,\\
&2e_i^{\mathsf T}P_rB_i\Delta g_i(t-h(t))
\le\varepsilon_{i,r}\|e_i\|^2
+\varepsilon_{i,r}^{-1}\|P_rB_i\|^2l_g^2\|e_{i,h}\|^2.
\end{align*}

Apply the Leibniz rule to the single-integral term. Using $\dot h(t)\le\delta$, we obtain
\begin{align*}
&D^+\!\left[
\int_{t-h(t)}^t
e^{\mathsf T}(v)(I_N\otimes S_r)e(v)\,\dd v
\right]\le
\lambda_{\max}(S_r)\sum_{i=1}^N\|e_i\|^2
-(1-\delta)\lambda_{\min}(S_r)
\sum_{i=1}^N\|e_{i,h}\|^2.
\end{align*}

The Dini derivative of the double-integral term satisfies
\begin{align*}
&D^+\!\left[
\int_{-\bar h}^{0}\int_{t+\theta}^{t}
e^{\mathsf T}(v)(I_N\otimes R_r)e(v)\,\dd v\,\dd\theta
\right]\\
&\quad=
\bar h\,e^{\mathsf T}(t)(I_N\otimes R_r)e(t)
-\int_{t-\bar h}^{t}
e^{\mathsf T}(v)(I_N\otimes R_r)e(v)\,\dd v\\
&\quad\le
\bar h\lambda_{\max}(R_r)\sum_{i=1}^N\|e_i\|^2.
\end{align*}

Combining the above estimates gives the following bound for the Dini derivative of the full functional:
\begin{equation}
\begin{aligned}
D^+\mathcal V_r^1(t)
\le\sum_{i=1}^N\bigl[
 e_i^{\mathsf T}\Lambda_{i,r}^{(11)}e_i
 +e_{i,h}^{\mathsf T}\Lambda_{i,r}^{(22)}e_{i,h}
\bigr].
\end{aligned}
\label{eq:Vfull-derivative}
\end{equation}

By \eqref{eq:flow-lmi},
\begin{equation}
  D^+\mathcal V_r^1(t)
  \le\lambda_{1,r}e^{\mathsf T}(t)(I_N\otimes P_r)e(t)
  \le\lambda_{1,r}\mathcal V_r^1(t).
  \label{eq:normal-growth}
\end{equation}

Suppose that a switch from mode $r$ to mode $j$ occurs during a normal phase. Applying \eqref{eq:mode-comparison} term by term to \eqref{eq:Vfull} gives
\begin{equation}
  \mathcal V_j^1(t^+)\le\mu\mathcal V_r^1(t^-).
  \label{eq:switch-jump}
\end{equation}

We next estimate the functional transition at the end of a recovery window. Let $t_k^\star:=t_k+\bar h$. Since $h(t_k^\star)\le\bar h$, all history involved in \eqref{eq:Vfull} at $t_k^\star$ lies in the recovery window $[t_k,t_k^\star]$. By \eqref{eq:two-sided-gronwall}, every $v$ in this interval satisfies
\begin{equation}
  \mathcal V_r^0(v)
  \le e^{(\beta_r+\gamma_r)\bar h}
  \mathcal V_r^0(t_k^\star-).
  \label{eq:history-compare}
\end{equation}

Moreover,
$\|e(v)\|^2\le\mathcal V_r^0(v)/\lambda_{\min}(P_r)$. Hence,
\begin{align*}
&\int_{t_k^\star-h(t_k^\star)}^{t_k^\star}
e^{\mathsf T}(v)(I_N\otimes S_r)e(v)\,\dd v\le
e^{(\beta_r+\gamma_r)\bar h}
\frac{\bar h\lambda_{\max}(S_r)}{\lambda_{\min}(P_r)}
\mathcal V_r^0(t_k^\star-),\\
&\int_{-\bar h}^{0}\int_{t_k^\star+\theta}^{t_k^\star}
e^{\mathsf T}(v)(I_N\otimes R_r)e(v)\,\dd v\,\dd\theta\le
e^{(\beta_r+\gamma_r)\bar h}
\frac{\bar h^2\lambda_{\max}(R_r)}{2\lambda_{\min}(P_r)}
\mathcal V_r^0(t_k^\star-).
\end{align*}

Therefore, the definition of $\widetilde\chi_r$ and \eqref{eq:mode-hold-right} yield
\begin{equation}
  \mathcal V_r^1((t_k^\star)^+)
  \le\widetilde\chi_r\mathcal V_r^0(t_k^\star-).
  \label{eq:phase-switch}
\end{equation}

By \eqref{eq:mode-hold} and \eqref{eq:mode-hold-right}, the mode
$r_k=\sigma(s_k^-)=\sigma(t_k^-)$ remains unchanged on $[s_k,t_k+\bar h]$. Using \eqref{eq:impulse-map}, \eqref{eq:impulse-contraction-condition}, and the dominance of the full functional over its current-state quadratic term, we have
\begin{align}
\mathcal V_{r_k}^0(t_k^+)
&=e^{\mathsf T}(s_k^-)(M_k^{\mathsf T}M_k\otimes P_{r_k})e(s_k^-)\le\eta_k e^{\mathsf T}(s_k^-)(I_N\otimes P_{r_k})e(s_k^-)
\le\eta_k\mathcal V_{r_k}^1(s_k^-).
\label{eq:impulse-contraction}
\end{align}

Successive application of \eqref{eq:V0-upper}, \eqref{eq:phase-switch}, and \eqref{eq:ck} gives
\begin{align}
\mathcal V_{r_k}^1((t_k+\bar h)^+)
&\le\widetilde\chi_{\max}e^{\beta_+\bar h}\eta_k
 \mathcal V_{r_k}^1(s_k^-)=e^{-\widetilde c_k}\mathcal V_{r_k}^1(s_k^-).
\label{eq:cycle-contraction}
\end{align}

By \eqref{eq:trigger-margin}, $\tau\ge\bar h$, and \eqref{eq:Phi},
\begin{equation}
  \Phi_k-\bar h
  =\frac{a_k-\rho\bar h+\bar\lambda(\tau-\bar h)}
  {\bar\lambda+\rho}>0.
  \label{eq:normal-length}
\end{equation}

Thus, $[t_k+\bar h,s_{k+1})$ is a nondegenerate normal phase. Let $N_k^{\rm sw}$ denote the number of switches in this interval. Combining \eqref{eq:normal-growth}, \eqref{eq:switch-jump}, and \eqref{eq:cycle-contraction}, we obtain
\begin{equation}
  \mathscr V(s_{k+1}^-)
  \le e^{\bar\lambda_1(\Phi_k-\bar h)-\widetilde c_k}
  \mu^{N_k^{\rm sw}}\mathscr V(s_k^-).
  \label{eq:sample-recursion}
\end{equation}

Condition \eqref{eq:mode-hold} excludes mode switches in all other intervals of the current cycle. Thus, $N_k^{\rm sw}$ is also the total number of switches in $[s_k,s_{k+1})$. Iterate \eqref{eq:sample-recursion} and use
$
  \sum_{i=0}^{k-1}\Phi_i=(s_k-s_0)-k\tau,
  \,
  \sum_{i=0}^{k-1}N_i^{\rm sw}
  \le\frac{s_k-s_0}{T_a}+N_0,
$
to obtain
\begin{align}
\ln\mathscr V(s_k^-)
\le{}&\ln\mathscr V(s_0^-)
 +\bar\lambda_1\bigl[(s_k-s_0)-k\tau-k\bar h\bigr]-\sum_{i=0}^{k-1}\widetilde c_i
 +\frac{\ln\mu}{T_a}(s_k-s_0)+N_0\ln\mu.
\label{eq:log-recursion}
\end{align}

Set $u=s_0$ and $t=s_k$ in \eqref{eq:cumulative-decay}. Then
\begin{equation}
 -\sum_{i=0}^{k-1}\widetilde c_i
 +\frac{\ln\mu}{T_a}(s_k-s_0)
 \le-(\alpha+\bar\lambda)(s_k-s_0)
 -k(\bar\lambda+\rho)\tau+\varphi.
 \label{eq:H5-sample}
\end{equation}

Substitute \eqref{eq:H5-sample} into \eqref{eq:log-recursion} and discard the nonpositive terms involving $k$. Let
$\varphi_{\rm eff}:=\varphi+N_0\ln\mu$. We obtain
\begin{equation}
  \mathscr V(s_k^-)
  \le e^{\varphi_{\rm eff}}
  \mathscr V(s_0^-)
  e^{-\alpha_{\rm eff}(s_k-s_0)}.
  \label{eq:sample-decay}
\end{equation}

We next estimate a uniform upper bound for the triggering intervals. Fix $k$ and take
$0<\varepsilon<\Phi_k$. Set $u=t_k+\varepsilon$ and $t=s_{k+1}$ in \eqref{eq:cumulative-decay}.
The interval $[t_k+\varepsilon,s_{k+1})$ contains no impulse and has length
$\Phi_k-\varepsilon$. Hence,
\[
  \left(\alpha+\bar\lambda+\frac{\ln\mu}{T_a}\right)
  (\Phi_k-\varepsilon)\le\varphi.
\]

Letting $\varepsilon\downarrow0$ gives
\begin{equation}
  0<\Phi_k\le\Phi_{\max}:=
  \frac{\varphi}{\alpha+\bar\lambda+\frac{\ln\mu}{T_a}}<\infty.
  \label{eq:Phi-max}
\end{equation}

Thus, $\Phi_k$ has the uniform upper bound required for the estimate over the entire time domain.

We now extend the estimate at the sampling instants to all times. If $t\in[s_k,t_k)$, then \eqref{eq:normal-growth} and \eqref{eq:mode-hold} give
\[
  \mathscr V(t)\le e^{\bar\lambda\tau}\mathscr V(s_k^-).
\]

If $t\in[t_k,t_k+\bar h)$, then \eqref{eq:V0-upper} and \eqref{eq:impulse-contraction} give
\[
  \mathscr V(t)\le e^{\beta_+\bar h}\mathscr V(s_k^-)
  \le e^{\bar\lambda\bar h}\mathscr V(s_k^-).
\]

If $t\in[t_k+\bar h,s_{k+1})$, then \eqref{eq:phase-switch}, \eqref{eq:normal-growth}, \eqref{eq:switch-jump}, \eqref{eq:adt}, and \eqref{eq:Phi-max} give
\[
\mathscr V(t)
\le\widetilde\chi_{\max}
e^{\beta_+\bar h+\bar\lambda_1\Phi_{\max}}
\mu^{N_0+\Phi_{\max}/T_a}\mathscr V(s_k^-).
\]

Therefore, for every $t\in[s_k,s_{k+1})$,
\begin{equation}
  \mathscr V(t)\le C_{\rm cyc}\mathscr V(s_k^-),
  \label{eq:cycle-bound}
\end{equation}
where
\begin{equation}
\begin{aligned}
C_{\rm cyc}:=\max\bigl\{&e^{\bar\lambda\tau},
 e^{\bar\lambda\bar h},\widetilde\chi_{\max}
e^{\beta_+\bar h+\bar\lambda_1\Phi_{\max}}
\mu^{N_0+\Phi_{\max}/T_a}\bigr\}.
\end{aligned}
\label{eq:Ccyc}
\end{equation}

Moreover, $t-s_k<\tau+\Phi_{\max}$. It follows from \eqref{eq:sample-decay} and \eqref{eq:cycle-bound} that
\begin{equation}
\mathscr V(t)
\le C_{\rm cyc}
e^{\varphi_{\rm eff}+\alpha_{\rm eff}(\tau+\Phi_{\max})}
\mathscr V(s_0^-)
e^{-\alpha_{\rm eff}(t-s_0)}.
\label{eq:all-time-V}
\end{equation}

Define
\begin{align*}
\underline p&:=\min_{r\in\mathcal S}\lambda_{\min}(P_r),\\
C_h&:=\max_{r\in\mathcal S}\lambda_{\max}(P_r)
+\bar h\max_{r\in\mathcal S}\lambda_{\max}(S_r)
+\frac{\bar h^2}{2}\max_{r\in\mathcal S}\lambda_{\max}(R_r).
\end{align*}

By \eqref{eq:V0} and \eqref{eq:Vfull},
\begin{equation}
  \mathscr V(t)\ge\underline p\|e(t)\|^2,
  \qquad
  \mathscr V(s_0^-)\le C_h\|e_{s_0}\|_{\bar h}^2.
  \label{eq:norm-equivalence}
\end{equation}

Substituting \eqref{eq:norm-equivalence} into \eqref{eq:all-time-V} gives \eqref{eq:global-exp}, where we can take
\begin{equation}
  C_{\rm GE}:=
  \left(\frac{C_{\rm cyc}C_h}{\underline p}\right)^{1/2}
  \exp\!\left[
  \frac{\varphi_{\rm eff}+\alpha_{\rm eff}
  (\tau+\Phi_{\max})}{2}\right].
  \label{eq:CGE}
\end{equation}

Since $e_i(t)=x_i(t)-s(t)$, \eqref{eq:global-exp} shows that every follower error $e_i(t)$ converges to zero at an exponential rate of at least $\alpha_{\rm eff}/2$. Therefore, the closed-loop system achieves global exponential leader-follower synchronization.
\end{proof}

\begin{remark}
When $\gamma_r^{(0)}\le0$ holds for every mode, $\gamma_r=0$. The phase-transition factor $\widetilde\chi_r$ then automatically reduces to its original form without the bidirectional growth compensation. This reduction leaves all other conditions of the theorem and the proof structure unchanged.
\end{remark}

\subsection{Numerical Verification}\label{sec:numerical-verification}

\begin{example}\label{ex:numerical-simulation}

Consider one three-dimensional leader and five three-dimensional heterogeneous
followers. The nonlinear functions are applied componentwise and are chosen as
$
f(z)=g(z)=\tanh(z),\, l_f=l_g=1.
$
The matrices of the leader system are
$
C=\diag(0.45,0.50,0.55),\,
B=\diag(0.12,0.10,0.09),
$
and
\[
A=\begin{bmatrix}
1.05&-0.25&0.15\\
0.20&0.95&-0.20\\
-0.15&0.25&1.00
\end{bmatrix}.
\]

The follower matrices are diagonal and are chosen as
\begin{align*}
C_1&=\diag(0.30,0.36,0.42),&
A_1&=\diag(1.20,0.95,1.05),&
B_1&=\diag(0.15,0.10,0.12),\\
C_2&=\diag(0.34,0.31,0.46),&
A_2&=\diag(1.10,1.18,0.92),&
B_2&=\diag(0.11,0.14,0.09),\\
C_3&=\diag(0.38,0.44,0.32),&
A_3&=\diag(0.98,1.08,1.15),&
B_3&=\diag(0.13,0.08,0.15),\\
C_4&=\diag(0.33,0.48,0.37),&
A_4&=\diag(1.16,1.02,1.10),&
B_4&=\diag(0.10,0.12,0.14),\\
C_5&=\diag(0.41,0.35,0.50),&
A_5&=\diag(1.05,1.12,1.19),&
B_5&=\diag(0.14,0.15,0.11).
\end{align*}

The continuous feedback gains are $K_i=1.20I_3$ for $i=1,\ldots,5$.

The time-varying delay is
$
h(t)=0.075+0.025\sin(0.8t),
$
so $0.05\le h(t)\le\bar h=0.10$ and
$\max_t\dot h(t)=0.02<\delta=0.05<1$.
The two switching topologies are a five-node ring and a star centered at
follower 3. Their Laplacian matrices are
\[
L_1=\begin{bmatrix}
2&-1&0&0&-1\\
-1&2&-1&0&0\\
0&-1&2&-1&0\\
0&0&-1&2&-1\\
-1&0&0&-1&2
\end{bmatrix},
\qquad
L_2=\begin{bmatrix}
1&0&-1&0&0\\
0&1&-1&0&0\\
-1&-1&4&-1&-1\\
0&0&-1&1&0\\
0&0&-1&0&1
\end{bmatrix}.
\]

Only follower 1 has a direct pinning channel, and hence
$D=\diag(1,0,0,0,0)$.

For the impulsive controller, choose
$
\mu_k=0.05,\eta_k=0.07, \tau=1.00,
 \bar\lambda=0,
 \rho=0.10,
 a_k=0.10$.

The switch in each cycle is scheduled at $s_k+1.50$ s, after the protected
interval ends at $s_k+\tau+\bar h=s_k+1.10$ s. The two impulsive topology modes
alternate over successive triggering cycles. The switching parameters are
$T_a=2.00$ and $N_0=1$. The Lyapunov parameters are selected as
$
P_r=I_3, S_r=0.25I_3, R_r=0.05I_3,
\varepsilon_{i,r}=0.10,
$
for both modes, with
$
\lambda_{1,r}=0, \mu=1,
\alpha=0.80, \varphi=1.97.
$

The initial histories are constant on $[-0.10,0]$. They are specified by
\[
s(\theta)=\begin{bmatrix}0.60&-0.40&0.80\end{bmatrix}^{\mathsf T},
\qquad
\begin{bmatrix}
e_1^{\mathsf T}(\theta)\\
e_2^{\mathsf T}(\theta)\\
e_3^{\mathsf T}(\theta)\\
e_4^{\mathsf T}(\theta)\\
e_5^{\mathsf T}(\theta)
\end{bmatrix}
=
\begin{bmatrix}
5&1&-3\\
4&2&-2\\
-5&-2&0\\
1&0&-1\\
-7&-3&1
\end{bmatrix},
\quad -0.10\le\theta\le0.
\]

These values satisfy all conditions of
Theorem~\ref{thm:main}. Therefore, the network achieves global exponential
leader-follower synchronization under the proposed hybrid control. Figure~\ref{fig:controlled-errors} shows that all synchronization errors
converge to zero under the hybrid control. The maximum error at $t=12$ s is
$4.3163\times10^{-7}$. For comparison, all control and inter-node coupling are
removed in Figure~\ref{fig:uncontrolled-errors}. The error curves remain
separated, and the maximum error at $t=12$ s is approximately $8.4527$.

\begin{figure}[pos=!htbp]
\centering
\includegraphics[width=0.92\linewidth]{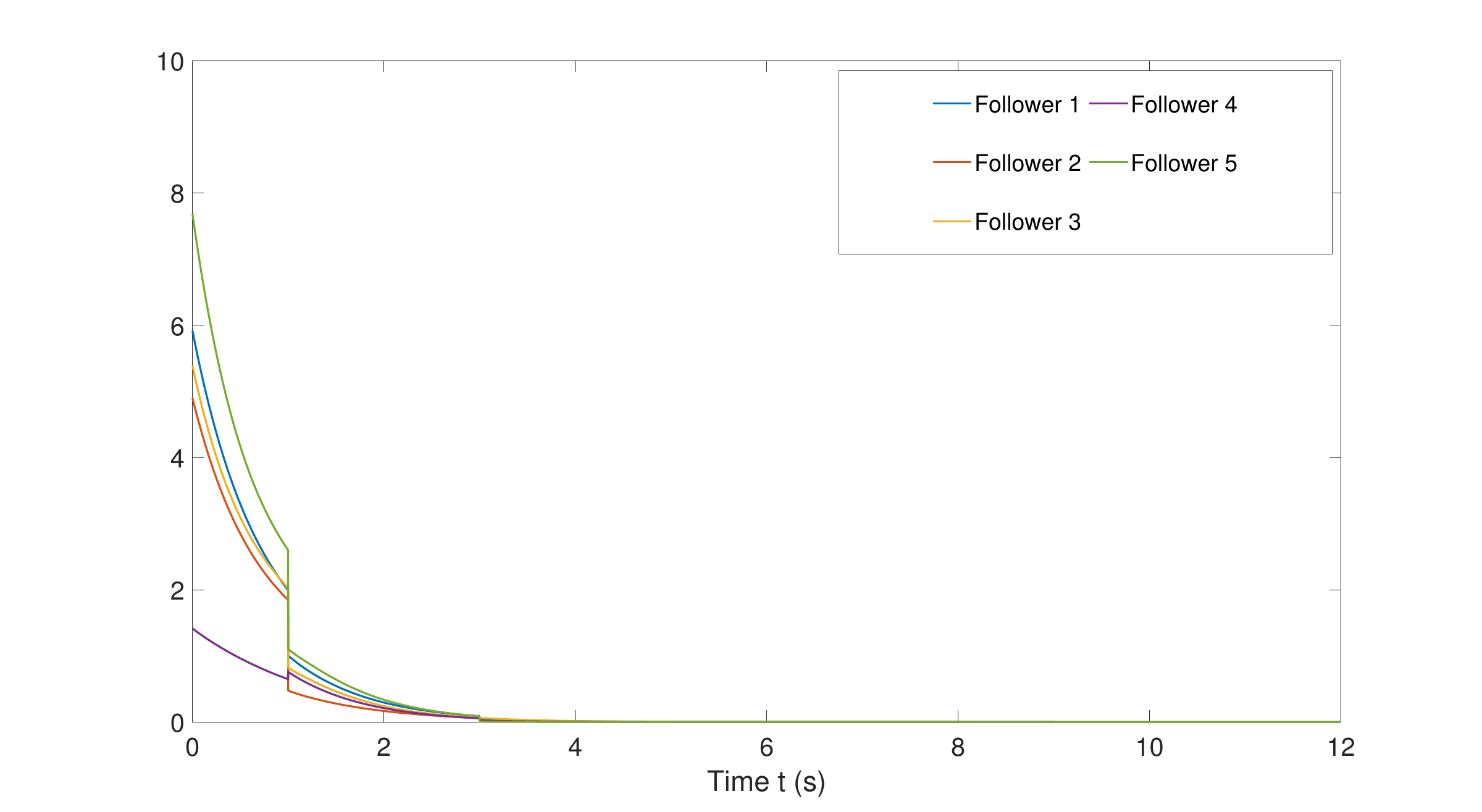}
\caption{Synchronization errors $\|x_i(t)-s(t)\|_2$ of the five followers under the proposed hybrid control.}
\label{fig:controlled-errors}
\end{figure}

\begin{figure}[pos=!htbp]
\centering
\includegraphics[width=0.92\linewidth]{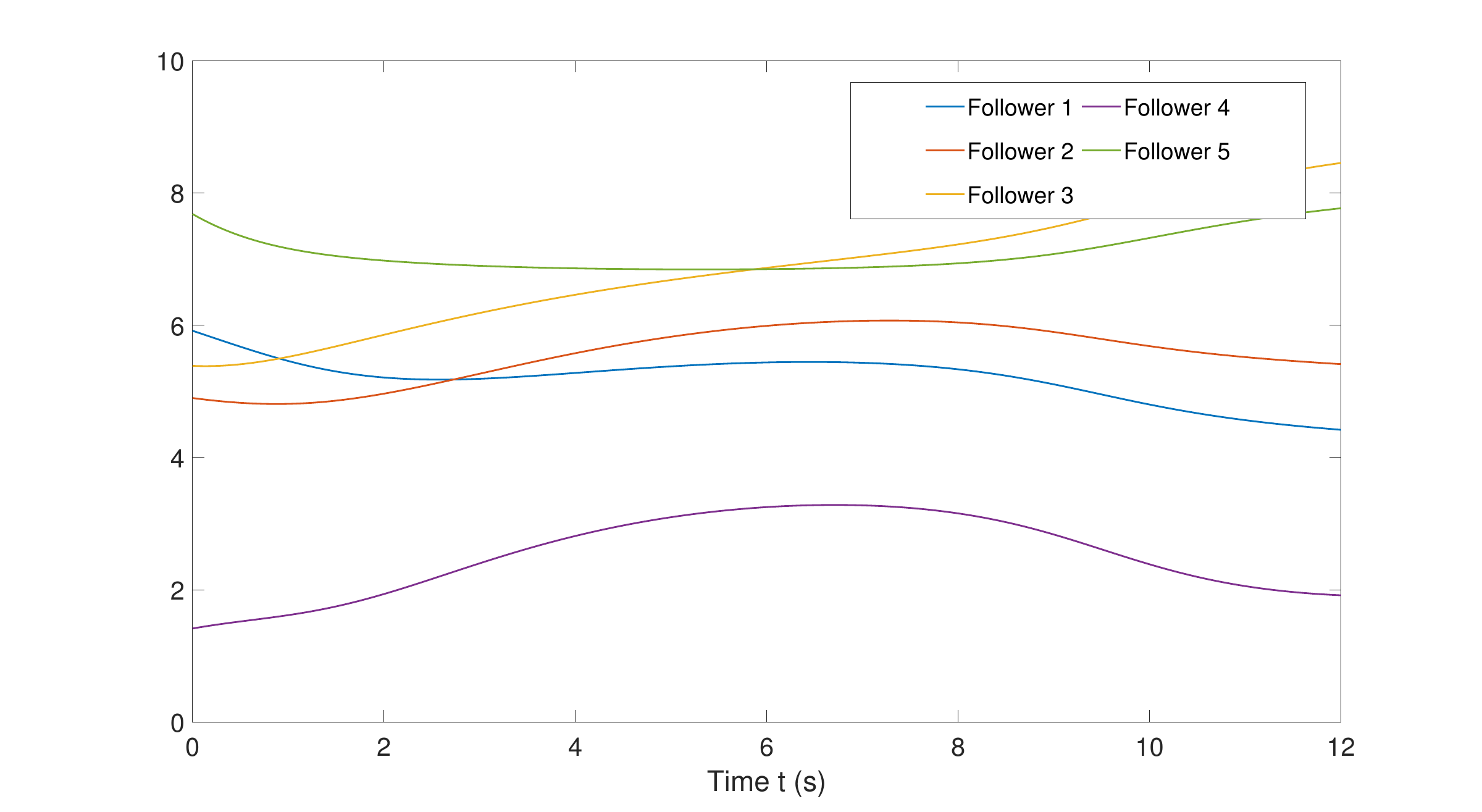}
\caption{Synchronization errors $\|x_i(t)-s(t)\|_2$ of the five followers without control.}
\label{fig:uncontrolled-errors}
\end{figure}
\end{example}

\section{Conclusion}

Andy provides an end-to-end workflow that turns a mathematical problem into a traceable proof. It organizes proof steps in an executable DAG, binds each local verification result to a certificate, uses node interfaces to confine repairs to the affected proof region, and records the complete path from problem formulation to final proof. The system separates proof generation from correctness evaluation and can acquire, retain, retrieve, and reuse knowledge from existing results. The case study shows that Andy can develop a technically meaningful problem from an established result and carry it through theorem construction, proof, and numerical illustration.

Starting from the self-triggered impulsive consensus result in \cite{hong2022consensus}, Andy formulated a global exponential leader-follower synchronization problem for delayed heterogeneous networks with switching communication topologies. The control-theoretic contribution is the recovery-window design. After each delayed impulse, model-matching feedback temporarily removes the delayed error channel until the pre-impulse history leaves the active delay interval. This separation permits a quadratic Lyapunov estimate during recovery and restores the full Lyapunov-Krasovskii functional afterward. It brings switching topology, impulse execution delay, and self-triggered updates into one exponential-synchronization proof. Sufficient conditions were established, Zeno behavior was excluded for both timing sequences, and the numerical example illustrates the expected synchronization behavior.

Future work will investigate CrewAI-style role-based multi-agent orchestration in which the solver, evaluator, and human-like monitor operate as distinct agents. Independent roles can reduce correlated self-evaluation errors and make disagreements explicit. Parallel exploration can increase proof-route coverage and shorten the time spent on unsuccessful branches. Role-specific messages and certificates can also improve traceability, while dynamic routing, escalation, and cross-checking can make the workflow more robust to errors from a single model or reasoning route.

\FloatBarrier

\bibliographystyle{elsarticle-num}
\bibliography{andy_refs}

\appendix
\section{Research Trace of the Case Study}
\label{app:research-trace}

This appendix records the main research decisions made during the case study. It summarizes the problem-revision history, the low-cost exploration of proof routes, the construction of the proof dependency graph, and the selection of the final proof chain. 

\subsection{Problem Revision and Research-Value Evaluation}

The evaluator recorded nine consecutive versions of the proposed research problem, indexed from 0 to 8. Thus, the initial proposal was followed by eight revision cycles. Each version was scored for importance, originality, feasibility, and coherence. Acceptance required a weighted total of at least $7.5$, a score of at least $6$ in every dimension, and a single research objective that could be answered by one principal conclusion. Table~\ref{tab:problem-revisions} gives the complete evaluation trace.

\begin{table}[pos=!htbp]
\caption{Evaluation trace for the proposed research problem.}
\label{tab:problem-revisions}
\centering
\small
\setlength{\tabcolsep}{7pt}
\begin{tabular}{@{}c c c c c c@{}}
\toprule
Version & Imp. & Orig. & Feas. & Coh. & Total \\
\midrule
0 & 6.5 & 6 & 6 & 5.5 & 6.07 \\
1 & 6 & 5 & 3 & 4 & 4.65 \\
2 & 7 & 7 & 6 & 7 & 6.75 \\
3 & 7 & 8 & 6 & 7 & 7.05 \\
4 & 8 & 8 & 7 & 8 & 7.75 \\
5 & 8 & 8 & 7 & 8 & 7.75 \\
6 & 8 & 8 & 7 & 8 & 7.75 \\
7 & 8 & 8 & 7 & 8 & 7.75 \\
8 & 8 & 8 & 7 & 8 & 7.75 \\
\bottomrule
\end{tabular}
\end{table}

\begin{example}

In Version 1, Andy proposed the following problem. Consider a leader and $N$ heterogeneous followers governed by
\begin{align*}
\dot s(t)&=-Cs(t)+Af(s(t))+Bg(s(t-h(t))),\\
\dot x_i(t)&=-C_ix_i(t)+A_if(x_i(t))+B_ig(x_i(t-h(t)))+u_i^c(t)+u_i^p(t),
\end{align*}
with $e_i=x_i-s$, use the manifold-compensation feedback $u_i^c=(C_i-C)s+(A-A_i)f(s)+(B-B_i)g(s(t-h(t)))-K_ie_i$. At each delayed impulse time $t_k=s_k+\tau$, impose $e(t_k^+)=(M_k\otimes I_n)e(s_k^-)$. For the switching quadratic function $V(t)=e^{\mathsf T}(t)(I_N\otimes P_{\sigma(t)})e(t)$, introduce the scalar comparison system $\dot W=\bar\lambda_1W+\kappa$ between impulses and $W(t_k)=\eta_kW(s_k^-)$ at impulses. Generate the next sampling time from
\[
s_{k+1}=\inf\left\{t>s_k+\tau:t-s_k-\tau-\Phi(k,W_k^2)\ge0\right\},
\qquad
\Phi(k,W_k^2)=\frac{a_k+\ln W_k^2}{\bar\lambda_1+\rho}.
\]
The task was to select the controller and triggering parameters so that the switching, impulse-contraction, and cumulative-decay conditions guaranteed global exponential leader-follower synchronization and a uniform positive lower bound on the intersampling intervals.

This version received scores of $6$, $5$, $3$, and $4$ for importance, originality, feasibility, and coherence, respectively, with a total score of $4.65$. Feasibility was the principal failure. The proposed argument claimed that a fixed constant $\kappa$ could absorb the delayed state-dependent term $\lambda_2W(t-h(t))$. Such a constant cannot dominate that term while the same comparison system is required to converge to zero. The Halanay-Li comparison bridge therefore did not close. The impulsive condition also constrained only the instantaneous quadratic term and did not control the history terms of a complete Lyapunov-Krasovskii functional. Coherence fell below the threshold because the delay, switching, and impulsive mechanisms were not connected by one valid estimate, and the required condition $\dot h(t)\le\delta<1$ had not been stated. The originality score was also below the threshold because the claimed integration of the mechanisms depended on this invalid step and two literature attributions were incomplete.

The evaluator recommended removing the additive-constant absorption step, constructing an explicit mode-dependent Lyapunov-Krasovskii functional, adding the delay-derivative assumption, deriving verifiable block matrix inequalities, and treating the jump of the complete functional at impulsive instants. It also requested a concrete intermediate lemma that connected the flow, delay, and impulse estimates before the revised problem proceeded to full proof execution.
\end{example}

Versions 4 through 7 had already cleared every numerical threshold. They were rejected solely because the proposal still presented several parallel subquestions as separate targets. The final revision retained the mathematical content and reorganized these requests into one objective: establish global exponential leader-follower synchronization under a unified switching hybrid-control theorem.

\subsection{Low-Cost Route Exploration and DAG Construction}

The route planner first explored several strategies at low computational cost. The first round considered scalar comparison domination, direct interval-wise Lyapunov stitching, and a discrete impulse-map approach. The second round refined the comparison route into a Li-type scalar comparison argument with a cumulative condition and also considered multiple-Lyapunov stitching, unified growth-rate absorption, and a direct impulsive Gronwall-Halanay recurrence. These two rounds produced seven route records. The comparison route received the highest feasibility assessment and became the main route. The strongest structurally different candidates were retained as backups. The discrete-map and unified-growth approaches were discarded because they either lost essential delay and switching information or introduced excessive conservatism. The seven records therefore represent broad preliminary exploration, while only three route families were promoted to executable branches.

\subsection{Proof-Node Responsibilities and Exact Dependencies}

Here and below, $N_i$ denotes the $i$th proof node in the execution plan. Each proof node is a mathematical work unit with a specified input, conclusion, dependency set, and verification criterion. The subscript $i$ is a stable node identifier; model iterations, difficulty levels, and branch numbers use different notation. The notation $N_i\text{-}r_j$ denotes the immutable $j$th revision of node $N_i$. Table~\ref{tab:proof-nodes} defines every node.

\begin{table}[pos=!htbp]
\caption{Proof-node responsibilities, direct dependencies, and final records.}
\label{tab:proof-nodes}
\centering
\footnotesize
\setlength{\tabcolsep}{3pt}
\renewcommand{\arraystretch}{1.10}
\begin{tabular}{@{}c p{0.45\linewidth} p{0.17\linewidth} p{0.17\linewidth}@{}}
\toprule
Node & Responsibility & Direct predecessors & Record \\
\midrule
$N_1$ & Two-stage error system, recovery-window controller, and impulsive error map & None & Frozen at $r_2$ \\
$N_2$ & Recovery-window quadratic growth bound & $N_1$ & Frozen at $r_1$ \\
$N_3$ & Normal-phase LKF Dini-derivative estimate & $N_1$ & Frozen at $r_1$ \\
$N_4$ & Impulse jump estimate and effective contraction factor & $N_1$ & Frozen at $r_1$ \\
$N_5$ & Interface between the recovery and normal phases & $N_1,N_2$ & Frozen at $r_5$ \\
$N_6$ & Li-type scalar comparison dominating the full functional & $N_2,N_3,N_4,N_5$ & Invalidated \\
$N_7$ & Cauchy factor and exponential-decay estimate & $N_6$ & Invalidated \\
$N_8$ & Main-branch theorem assembly & $N_7$ & Invalidated \\
$N_9$ & Direct interval-wise stitching without a scalar comparator & $N_2,N_3,N_4,N_5$ & Failed after $r_3$ \\
$N_{10}$ & First backup-theorem assembly & $N_9$ & Not run; invalidated \\
$N_{11}$ & Direct impulsive Gronwall-Halanay cumulative recurrence & $N_2,N_3,N_4,N_5$ & Frozen at $r_2$ \\
$N_{12}$ & Final synchronization-theorem assembly & $N_{11}$ & Frozen at $r_4$ \\
\bottomrule
\end{tabular}
\end{table}

\FloatBarrier

The common prefix of the graph contained the system reduction and the four estimates needed by every branch. The planner then attached three alternative tails to this verified prefix. Figure~\ref{fig:case-study-dag} shows the resulting execution graph.

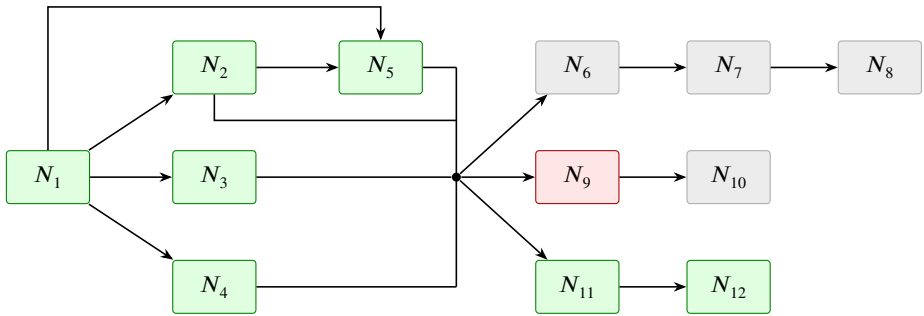
\begin{figure}[pos=!htbp]
\centering
\begin{tikzpicture}[
proofnode/.style={draw, rounded corners=1.3pt, align=center, minimum width=11mm,
  minimum height=7mm, font=\small, inner sep=1.4pt},
frozen/.style={proofnode, fill=green!12, draw=green!50!black},
failed/.style={proofnode, fill=red!10, draw=red!65!black},
invalid/.style={proofnode, fill=gray!15, draw=gray!65},
edge/.style={-{Stealth[length=1.7mm]}, semithick},
shared/.style={semithick},
junction/.style={circle, fill=black, inner sep=1.2pt}
]
\node[frozen] (n1) at (0,0) {$N_1$};
\node[frozen] (n2) at (2.2,1.45) {$N_2$};
\node[frozen] (n3) at (2.2,0) {$N_3$};
\node[frozen] (n4) at (2.2,-1.45) {$N_4$};
\node[frozen] (n5) at (4.4,1.45) {$N_5$};
\coordinate (bustop) at (5.4,1.45);
\node[junction] (join) at (5.4,0) {};
\coordinate (busbottom) at (5.4,-1.45);

\node[invalid] (n6) at (7,1.45) {$N_6$};
\node[invalid] (n7) at (9,1.45) {$N_7$};
\node[invalid] (n8) at (11,1.45) {$N_8$};
\node[failed] (n9) at (7,0) {$N_9$};
\node[invalid] (n10) at (9,0) {$N_{10}$};
\node[frozen] (n11) at (7,-1.45) {$N_{11}$};
\node[frozen] (n12) at (9,-1.45) {$N_{12}$};

\draw[edge] (n1) -- (n2);
\draw[edge] (n1) -- (n3);
\draw[edge] (n1) -- (n4);
\draw[edge] (n1.north) -- (0,2.25) -- (4.4,2.25) -- (n5.north);
\draw[edge] (n2) -- (n5);
\draw[shared] (bustop) -- (busbottom);
\draw[shared] (n5.east) -- (bustop);
\draw[shared] (n2.south) -- (2.2,0.75) -- (5.4,0.75);
\draw[shared] (n3.east) -- (join);
\draw[shared] (n4.east) -- (busbottom);
\draw[edge] (join) -- (n6);
\draw[edge] (join) -- (n9);
\draw[edge] (join) -- (n11);
\draw[edge] (n6) -- (n7);
\draw[edge] (n7) -- (n8);
\draw[edge] (n9) -- (n10);
\draw[edge] (n11) -- (n12);
\end{tikzpicture}
\caption{Executed proof DAG. Node responsibilities and direct predecessors are listed in Table~\ref{tab:proof-nodes}. Green nodes were verified and frozen, red marks the failed backup node, and gray nodes were invalidated after an upstream change or branch failure.}
\label{fig:case-study-dag}
\end{figure}

\subsection{Failed Branches and Final Route Selection}

The main branch comprised $N_6$, $N_7$, and $N_8$. During the repair of $N_5$, the interface factor was replaced by the corrected mode-dependent quantity $\widetilde\chi_r$. This change altered the dependency fingerprint of the downstream comparison construction. The three nodes were therefore invalidated before further execution. Their invalidation records an upstream interface change and leaves the local validity of individual statements in the abandoned branch undecided.

The first backup branch failed at $N_9$ after three revisions. The first revision contained a sign error in the logarithmic exponent and counted the average-dwell-time chatter term twice. The second introduced an undeclared maximum interval and omitted a time-dependent switching factor. The third relied on an unavailable growth-rate inequality, counted the recovery cost twice, and omitted switches between a sampling instant and its delayed impulse. These failures showed that direct interval stitching was too fragile for the coupled recovery, delay, and switching interfaces. Consequently, $N_{10}$ was never executed.

The second backup branch retained the interval estimates and changed only their assembly mechanism. The first revision of $N_{11}$ omitted the corrected phase condition and did not verify that the normal interval had nonnegative length. Its second revision used $\widetilde\chi_{\max}$, added the required trigger margin, and completed the cumulative recurrence. The first three revisions of $N_{12}$ still exposed mismatches between the theorem statement and the verified proof interface. Revision $4$ made the actual assumptions explicit and passed independent verification.

The final proof chain was therefore
\[
N_1\text{-}r_2
\longrightarrow
\{N_2\text{-}r_1,N_3\text{-}r_1,N_4\text{-}r_1\}
\longrightarrow
N_5\text{-}r_5
\longrightarrow
N_{11}\text{-}r_2
\longrightarrow
N_{12}\text{-}r_4.
\]
The braces indicate that $N_2$, $N_3$, and $N_4$ were verified in parallel. The dependency of $N_5$ was limited to $N_1$ and $N_2$, while $N_{11}$ used all four shared estimates $N_2$ through $N_5$. At the end of the run, seven nodes were frozen, one node had failed, and four nodes had been invalidated. The graph contained nineteen immutable node revisions, which preserved the complete repair history while allowing the verified prefix to be reused.

\FloatBarrier

\end{document}